\documentclass{article}

\usepackage{microtype}
\usepackage{graphicx}
\usepackage{subcaption}
\usepackage{booktabs}

\usepackage{hyperref}

\usepackage[accepted]{icml2018}

\usepackage{graphicx}
\usepackage{comment}
\usepackage{amsmath}
\usepackage{amssymb}

\usepackage{amssymb}
\usepackage{amsmath}
\usepackage{amsthm}
\usepackage{bm}
\usepackage{bbm}
\usepackage{cleveref}

\usepackage{tikz}
\usetikzlibrary{calc,trees,positioning,arrows,chains,shapes.geometric,%
  decorations.pathreplacing,decorations.pathmorphing,shapes,%
  matrix,shapes.symbols,fit,decorations,arrows.meta}

\usepackage{scalerel}[2016/12/29]

\definecolor{nice-red}{HTML}{E41A1C}
\definecolor{nice-orange}{HTML}{FF7F00}
\definecolor{nice-yellow}{HTML}{FFC020}
\definecolor{nice-green}{HTML}{4DAF4A}
\definecolor{nice-blue}{HTML}{377EB8}
\definecolor{nice-purple}{HTML}{984EA3}

\usepackage{bm}

\usepackage{color}

\usepackage{url}

\usepackage{supertabular}

\numberwithin{equation}{section}

\numberwithin{thm}{section}

\DeclareMathAlphabet{\mathsfsl}{OT1}{cmss}{m}{sl}

\renewcommand{\phi}{\varphi}

\newcommand{\grad}{\nabla}

\newcommand{\gradtheta}{\grad_{\theta}}

\newcommand{\SL }{\text{SL}}
\newcommand{\eval}{\rightarrowtail}
\newcommand{\obj}{\mathcal{L}}
\newtheorem{theorem}{Theorem}
\newcommand{\wc}{\mathcal{W}_c}
\newcommand{\wcn}{\mathcal{W}_{c^{n}}}
\newcommand{\W}{\mathcal{W}}

\newcommand{\bernoulli}{\mathrm{Ber}}

\newcommand{\E}{\mathbb{E}}

\newcommand{\be}{\begin{equation}}
\newcommand{\ee}{\end{equation}}
\newcommand{\bea}{\begin{eqnarray}}
\newcommand{\eea}{\end{eqnarray}}
\newcommand{\beaa}{\begin{eqnarray*}}
\newcommand{\eeaa}{\end{eqnarray*}}

\DeclareMathAlphabet{\mathpzc}{OT1}{pzc}{m}{n}

\def\magicbox{\textsc{DiCE}}
\def\magicname{\textsc{MagicBox}}
\usepackage{magic_dice}
\def\stopgrad{\ensuremath{\bot}}
\def\ie{\emph{i.e.}}
\def\eg{\emph{e.g.}}

\def\-{ }

\icmltitlerunning{DiCE: The Infinitely Differentiable Monte Carlo Estimator}

\begin{document}

\twocolumn[
\icmltitle{DiCE: The Infinitely Differentiable Monte Carlo Estimator}

\icmlsetsymbol{equal}{*}

\begin{icmlauthorlist}
\icmlauthor{Jakob Foerster}{ox}
\icmlauthor{Gregory Farquhar}{equal,ox}
\icmlauthor{Maruan Al-Shedivat}{equal,cmu}\\
\icmlauthor{Tim Rockt{\"a}schel}{ox}
\icmlauthor{Eric P. Xing}{cmu}
\icmlauthor{Shimon Whiteson}{ox}
\end{icmlauthorlist}

\icmlaffiliation{ox}{University of Oxford}
\icmlaffiliation{cmu}{Carnegie Mellon University}

\icmlcorrespondingauthor{\\Jakob Foerster}{jakob.foerster@cs.ox.ac.uk}

\icmlkeywords{Machine Learning, ICML}

\vskip 0.3in
]

\printAffiliationsAndNotice{\icmlEqualContribution}

\begin{abstract}
The score function estimator is widely used for estimating gradients of stochastic objectives in \emph{stochastic computation graphs} (SCG), \eg{}, in reinforcement learning and meta-learning.
While deriving the first\-{}order gradient estimators by differentiating a 
\emph{surrogate loss} (SL) objective is computationally and conceptually 
simple, using the same approach for higher\-{}order derivatives is more 
challenging.
Firstly, analytically deriving and implementing such estimators is laborious and not compliant with automatic differentiation.
Secondly, repeatedly applying SL to construct new objectives for each order 
derivative involves increasingly cumbersome graph manipulations. 
Lastly, to match the first\-{}order gradient under differentiation, SL treats 
part of the cost as a fixed sample, which we show leads to missing and wrong 
terms for estimators of higher\-{}order derivatives.
To address all these shortcomings in a unified way, we introduce \magicbox{}, 
which provides a single objective that can be differentiated repeatedly, 
generating correct estimators of derivatives of any order in SCGs.
Unlike SL, \magicbox{} relies on automatic differentiation for performing the 
requisite graph manipulations.
We verify the correctness of  \magicbox{} both through a proof and 
numerical evaluation of the \magicbox{} derivative estimates. 
We also use \magicbox{} to propose and evaluate a novel approach for multi-agent learning. Our code is available at \href{https://www.github.com/alshedivat/lola}{github.com/alshedivat/lola}.

\end{abstract}

\section{Introduction}
The score function trick is used to produce Monte Carlo estimates of
gradients in settings with non-differentiable objectives, \eg, in meta-learning and reinforcement
learning.
Estimating the first order gradients is computationally and conceptually simple.
While the gradient estimators can be directly defined, it is often more
convenient to define an objective
whose derivative is the gradient estimator.
Then the automatic-differentiation (auto-diff)
toolbox, as implemented in deep learning libraries, can easily compute the 
gradient estimates with respect to all upstream parameters.

This is the
method used by the \emph{surrogate loss} (SL)
approach~\citep{schulman2015gradient}, which provides a recipe for building a
surrogate objective from a \emph{stochastic computation graph} (SCG).
When differentiated, the SL yields an estimator
for the first\-{}order gradient of the original objective.

However, estimating higher order derivatives is more challenging.  Such 
estimators are useful for a number of optimization
techniques, accelerating convergence in supervised settings
\citep{dennis1977quasi} and reinforcement learning 
\citep{furmston2016approximate}.
Furthermore, they are vital for gradient-based meta-learning
\citep{finn2017,al2017continuous,li2017meta}, which differentiates an objective
after some number of first order learning steps.
Estimators of higher order derivatives have also proven useful
in multi-agent learning \citep{foerster2018}, when one agent differentiates
through the learning process of another agent.

Unfortunately, the first order gradient estimators mentioned above are
fundamentally ill suited to calculating higher order derivatives via
auto-diff.
Due to the dependency on the sampling distribution, estimators of 
higher\-{}order derivatives
require repeated application of the score function trick.
Simply differentiating the first\-{}order estimator again, as was for
example done by \citet{finn2017}, leads to missing terms.

To obtain higher\-{}order score function estimators, there are
currently two unsatisfactory options.  The first is to analytically derive and
implement the estimators.
However, this is laborious, error prone, and does not comply with the auto-diff paradigm.
The second is to repeatedly apply the SL approach to construct new
objectives for each further derivative estimate. However, 
each of these new objectives involves increasingly complex graph manipulations,
defeating the purpose of a differentiable surrogate loss.

Moreover,
to match the first\-{}order gradient after a single differentiation, the SL treats
part of the cost as a fixed sample, severing the dependency on the parameters.  
We show that this yields missing and incorrect terms in estimators of 
higher\-{}order derivatives.
We believe that these difficulties have limited
the usage and exploration of higher\-{}order methods in reinforcement
learning tasks and other application areas that may be formulated as SCGs.

Therefore, we propose a novel technique, the \emph{Infinitely \textbf{Di}fferentiable Monte-\textbf{C}arlo \textbf{E}stimator} (\magicbox{}), to address all these
shortcomings. \magicbox{} constructs a single objective that evaluates
to an estimate of the original objective, but can also be differentiated
repeatedly to obtain correct estimators of derivatives of any order.
Unlike the SL approach, \magicbox{} relies on auto-diff
as implemented for instance in TensorFlow \citep{abadi2016tensorflow} or
PyTorch \citep{paszke2017automatic} to automatically perform the complex graph
manipulations required for these estimators of higher\-{}order derivatives.

\magicbox{} uses a novel operator, \magicname{} ($\magic$), that acts on the 
set of those stochastic nodes $\wc$ that influence each of the original losses 
in an SCG.
Upon differentiation, this operator generates the correct derivatives 
associated with the sampling distribution:
\[
  \grad_\theta\magic (\wc) =
  \magic (\wc)\grad_\theta\sum_{w\,\in\,\wc}\log(p(w;\theta)),
\]
while returning $1$ when evaluated: $\magic (\W) \eval 1$.
The  \mbox{\magicname}-operator can easily be implemented in standard deep learning libraries as
follows:
\begin{align*}
 \magic (\W) &= \exp\big(\tau - \stopgrad( \tau )\big), \\
\tau  &= \sum_{w \in \W} { \log(p(w;\theta))},
\end{align*}
where $\stopgrad$ is an operator that sets the gradient of the
operand to zero, so $\grad_x \text{\stopgrad}(x) = 0$. 
In addition, we show how to use a baseline for variance reduction in our 
formulation.

We verify the correctness of  \magicbox{} both through a proof and through 
numerical evaluation of the \magicbox{} gradient estimates. 
To demonstrate the utility of \magicbox{}, we also propose a novel approach for 
learning with opponent learning awareness \citep{foerster2018}.
We also open-source our code in TensorFlow.
We hope this powerful and convenient novel objective will unlock further exploration and adoption of higher\-{}order learning methods in meta-learning, reinforcement learning, and other applications of SCGs. Already, \magicbox{} is used to implement repeatedly differentiable gradient 
estimators with pyro.infer.util.Dice and tensorflow\_probability.python.monte\_carlo.expectation.

\section{Background}
\newcommand{\deterministic}{\ensuremath{\mathcal{D}}}
\newcommand{\stochastic}{\ensuremath{\mathcal{S}}}
\newcommand{\parameters}{\ensuremath{\Theta}}
\newcommand{\costs}{\ensuremath{\mathcal{C}}}
\newcommand{\influences}{\ensuremath{\prec}}
\newcommand{\deps}{\ensuremath{\textsc{deps}}}
\newcommand{\causes}{\ensuremath{\textsc{causes}}}
\newcommand{\costsum}{\ensuremath{\hat Q}}
\newcommand{\partheta}{\ensuremath{\frac{\partial}{\partial\theta}}}
\newcommand{\deriv}[1]{\ensuremath{\frac{\partial}{\partial#1}}}

\tikzstyle{input}=[]
\tikzstyle{stochastic}=[draw, circle, very thick, color=nice-orange!50!black, fill=nice-orange!10, minimum height=0.75cm, minimum width=0.75cm, text centered]
\tikzstyle{deterministic}=[draw, rectangle, very thick, minimum height=0.75cm, minimum width=0.75cm, text centered]
\tikzstyle{cost}=[deterministic, fill=black!10]
\tikzstyle{surrogate}=[cost, color=nice-blue!50!black, fill=nice-blue!10]
\tikzstyle{dep}=[very thick, -Latex]
\tikzstyle{grad}=[very thick, -Latex, color=nice-red!50!black]
\tikzstyle{sl}=[very thick, -Latex, dotted, color=nice-blue]
\tikzstyle{enchant}=[very thick, -Latex, dotted, color=nice-purple]
\tikzstyle{magic}=[ultra thick, draw, color=nice-purple, rounded corners]
\tikzstyle{magicbox}=[deterministic, color=nice-purple!50!black, fill=nice-purple!10]
\tikzstyle{gradbox}=[deterministic, color=nice-red!50!black, fill=nice-red!10]
\def\xdist{2.6}
\def\ydist{1}

Suppose $x$ is a random variable, $x \sim p(x; \theta)$, $f$ is a function of
$x$ and we want to compute $\gradtheta \E_x\left[f(x)\right]$. If the
analytical gradients $\gradtheta f$ are unavailable or nonexistent, we
can employ the \emph{score function} (SF) estimator \citep{fu2006gradient}:
\begin{equation}
\gradtheta \E_x\left[f(x)\right] = \E_x\left[f(x)\gradtheta \log(p(x; \theta)) \right]
\end{equation}
If instead $x$ is a deterministic function of $\theta$ and another random
variable $z$, the operators $\gradtheta$ and $\E_z$ commute, yielding the \emph{pathwise derivative estimator} or
\emph{reparameterisation trick} \citep{kingma2013auto}.
In this work, we focus on the SF estimator, which
can capture the interdependency of both the objective and the sampling
distribution on the parameters $\theta$, and therefore requires careful
handling for estimators of higher order derivatives.\footnote{In the following, we use the terms `gradient' and `derivative' interchangeably.}

\subsection{Stochastic Computation Graphs}
Gradient estimators for single random variables can be generalised using the
formalism of a stochastic computation graph \citep[SCG,][]{schulman2015gradient}.
An SCG is a directed acyclic graph with four types of nodes: \emph{input nodes}, $\parameters$; \emph{deterministic nodes}, $\deterministic$; \emph{cost nodes}, $\costs$; and \emph{stochastic nodes}, $\stochastic$.
Input nodes are set externally and can hold parameters we seek to optimise.
Deterministic nodes are functions of their parent nodes, while stochastic nodes
are distributions conditioned on their parent nodes. The set of
cost nodes $\costs$ are those associated with an objective $\obj = \E [ \sum_{c 
\in \costs} c]$.

Let $v \influences w$ denote that node $v$ \emph{influences} node $w$, \ie{},
there exists a path in the graph from $v$ to $w$.
If every node along the path is deterministic, $v$ influences $w$ deterministically which is denoted by $v \influences^D w$.
See Figure \ref{fig:bug} (top) for a simple SCG with an input node $\theta$, a
stochastic node $x$ and a cost function $f$.
Note that $\theta$ influences $f$ deterministically ($\theta \influences^D f$) as well as stochastically via $x$ ($\theta \influences f$).

\subsection{Surrogate Losses}
\label{ssec:sl}
In order to estimate gradients of a sum of cost nodes, $\sum_{c \in \costs} c$, in an arbitrary SCG, \citet{schulman2015gradient} introduce the notion of a \emph{surrogate loss} (SL):
\[
  \SL(\parameters,\stochastic) := \sum_{w\,\in\,\stochastic}\log p(w\ |\
  \deps_w)\costsum_w + \sum_{c\,\in\,\costs}c(\deps_c).
\]
Here $\deps_w$ are the `dependencies' of $w$: the set of stochastic or input
nodes that deterministically influence the node $w$.
Furthermore,  $\costsum_w$ is the sum of \emph{sampled} costs $\hat{c}$ 
corresponding to the cost nodes influenced by $w$.

The hat notation on $\costsum_w$ indicates that
inside the SL, these costs
are treated as fixed samples. This severs the functional dependency on $\theta$ 
that was present in the original stochastic computation graph.

The SL produces a gradient estimator when differentiated once \citep[Corollary 1]{schulman2015gradient}:
\begin{equation}
\gradtheta  \obj =  \E [ \gradtheta
\SL(\parameters,\stochastic) ].
\end{equation}
Note that the use of sampled costs $\costsum_w$ in the definition of the SL 
ensures that its first order 
gradients match the score function estimator, which does not contain 
a term of the form $\log(p) \gradtheta Q$.

Although  \citet{schulman2015gradient} focus on first order gradients, they argue that the SL gradient estimates themselves can be treated as costs in an SCG and that the SL approach can be applied repeatedly to construct higher order gradient estimators.
However, the use of sampled costs in the SL leads to missing dependencies and 
wrong estimates when calculating such higher order gradients, as we discuss in 
Section \ref{sec:slbug}.

\section{Higher Order Derivatives}
\label{sec:hog}
In this section, we illustrate how to estimate higher order derivatives via 
repeated application of the score function (SF) trick and show that repeated 
application of the surrogate loss (SL) approach in stochastic computation 
graphs (SCGs) fails to capture all of the relevant terms for higher order 
gradient estimates.

\subsection{Estimators of Higher Order Derivatives}
We begin by revisiting the derivation of the score function estimator for the gradient of the expectation $\obj$ of $f(x;\theta)$ over $x\sim p(x; \theta)$:
\begin{align}
\gradtheta \obj  &= \gradtheta  \E_x\left[f(x;\theta)\right]   \nonumber \\
 &= \gradtheta  \sum_x p(x; \theta) f(x;\theta)  \nonumber \\
&= \sum_x \gradtheta \big( p(x; \theta) f(x;\theta) \big) \nonumber  \\
&= \sum_x   \big(  f(x;\theta)  \gradtheta p(x; \theta) +   p(x; \theta) \gradtheta f(x;\theta) \big) \nonumber  \\
&= \sum_x    \big(   f(x;\theta)  p(x; \theta)  \gradtheta \log( p(x; \theta)) \nonumber  \\
&\qquad\quad +   p(x; \theta) \gradtheta f(x;\theta) \big) \nonumber  \\
&= \E_x\left[ f(x;\theta) \gradtheta \log(p(x; \theta))  +  \gradtheta
f(x;\theta)\right] \label{eq:exact_gradient} \\
&= \E_x[ g(x;\theta)]. \nonumber
\end{align}
The estimator $g(x;\theta)$ of the gradient of $\E_x\left[f(x;\theta)\right]$ consists of two distinct terms:
  (1) the term $f(x;\theta) \gradtheta \log(p(x; \theta))$ originating from $f(x;\theta) \gradtheta p(x; \theta)$ via the SF trick, and
  (2) the term  $\gradtheta f(x; \theta)$, due to the direct dependence of $f$ on $\theta$.
%
The second term is often ignored because $f$ is often only a function of $x$ but not of $\theta$.
However, even in that case, the gradient estimator $g$ depends on both $x$ and $\theta$.
We might be tempted to again apply the SL approach to $\gradtheta \E_x[ g(x;\theta)]$ to produce estimates of higher order gradients of $\obj$, but below we demonstrate that this fails. In Section~\ref{sec:MagicBox}, we introduce a practical algorithm for correctly producing such higher order gradient estimators in SCGs.

\subsection{Higher Order Surrogate Losses}

\begin{figure}[t]
  \centering
  \resizebox{\columnwidth}{!}{
    \begin{tikzpicture}
      \begin{scope}
        \node[anchor=west] at (-4.75,0.75) {\bf Stochastic Computation Graph};
        \node[] (t) at (0,0) {$\theta$};
        \node[stochastic, anchor=west] (x) at ($(t)+(1,0)$) {$x$};
        \node[cost, anchor=west] (f) at ($(x)+(1,0)$) {$f$};
        \foreach \from/\to in {t/x,x/f}
        \draw[dep] (\from) -- (\to);
        \draw[dep] (t) edge[bend left=40] (f);
      \end{scope}

      \begin{scope}
        \begin{scope}[shift={(0,-0.25)}]
          \draw[line width=3pt, black!10, dashed] (-5.75,-3.5) -- (7.5,-3.5);
          \node[anchor=west] at (-5.75,-4.25) {\Large $\grad_\theta \obj$};
          \draw[line width=3pt, black!10, dashed] (-5.75,-5) -- (7.5,-5);
          \node[anchor=west] at (-5.75,-5.75) {\Large $\grad^2_\theta \obj$};
          \draw[line width=3pt, black!10, dashed] (-5.75,-7.5) -- (7.5,-7.5);
          \node[anchor=west] at (-5.75,-8.25) {\Large $\grad^n_\theta \obj$};
        \end{scope}

        \draw[line width=3pt, black!20] ($(x)-(0,0.75)$) -- ($(x)-(0,3.75)$) -- ($(x)+(3.5,-3.75)$) -- ($(x)+(3.5,-9)$);
        \draw[line width=3pt, black!20] (-5,-0.75) -- (7.5,-0.75);
      \end{scope}

      \begin{scope}[shift={(-4.25,-2.5)}]
        \node[anchor=west] at (-0.5,1.25) {\bf Surrogate Loss Approach};
        \begin{scope}[shift={(0,-0.1)}]
          \node[] (t) at (0,0) {$\theta$};
          \node[stochastic, anchor=west] (x) at ($(t)+(1,0)$) {$x$};
          \node[cost, anchor=west] (f) at ($(x)+(1,0)$) {$f$};
          \foreach \from/\to in {t/x,x/f}
          \draw[dep] (\from) -- (\to);
          \draw[dep] (t) edge[bend left=40] (f);

          \node[surrogate, anchor=west] (sl1) at ($(t.west)-(0,1.95)$)
          {$\log(p(x;\theta))\hat f + f$};

          \node[gradbox, anchor=west] (gsl1) at ($(sl1.east)+(0.75,0)$)
          {$g_\SL = \hat f \grad_\theta\log(p(x;\theta)) +\grad_\theta f$};

          \node[surrogate, anchor=west, text width=3.5cm] (sl2) at ($(sl1.south
          west)-(0,1.25)$) {$\log(p(x;\theta))\hat g_\SL+\hat f
          \grad_\theta\log(p(x;\theta)) +\grad_\theta f$};

          \node[gradbox, anchor=north west, text width=4cm] (gsl2) at
          ($(sl2.north east)+(0.75,0)$) {$\hat g_\SL
          \grad_\theta\log(p(x;\theta))+\hat f \grad^2_\theta\log(p(x;\theta))
          + \grad_\theta\hat f\grad_\theta\log(p(x;\theta))  + \grad^2_\theta f$};

          \foreach \from/\to in {f/sl1,gsl1/sl2}
          \draw[sl] (\from) -- (\to);
          \foreach \from/\to in {sl1/gsl1,sl2/gsl2}
          \draw[grad] (\from.east) -- node[midway, anchor=north] {$\grad$} (\from.east -| \to.west);

          \draw[line width=3pt, red] ($(gsl2.south west)+(0.2,0.1)$) -- ($(gsl2.south east)+(-1.32,0.1)$);
          \node[red] at ($(gsl2.south west)+(1.5,-0.25)$) {\Large$\bf{=0}$};
        \end{scope}

      \end{scope}

      \begin{scope}[shift={(2,-2.5)}]
        \node[anchor=west] at (-0.5,1.25) {\bf \magicbox{}};
        \begin{scope}[shift={(0,-0.1)}]
          \node[] (t) at (0,0) {$\theta$};
          \node[stochastic, anchor=west] (x) at ($(t)+(1,0)$) {$x$};
          \node[cost, anchor=west] (f) at ($(x)+(1,0)$) {$f$};
          \node[magicbox, anchor=west] (m) at ($(f)+(1,0)$) {$\magic (x)f$};
          \node[gradbox, anchor=north] (m1) at ($(m.south)+(0,-1.15)$) {$\grad_\theta(\magic (x)f)$};
          \node[gradbox, anchor=north] (m2) at ($(m1.south)+(0,-0.75)$) {$\grad^2_\theta(\magic (x)f)$};
          \node[gradbox, anchor=north] (mn) at ($(m2.south)+(0,-1.75)$) {$\grad^n_\theta(\magic (x)f)$};

          \foreach \from/\to in {t/x,x/f}
          \draw[dep] (\from) -- (\to);
          \draw[dep] (t) edge[bend left=40] (f);
          \draw[dep] (f) -- (m);
          \draw[dep] (t) edge[bend left=40] (m);
          \draw[dep] (x) edge[bend right=40] (m);
          \foreach \from/\to in {m/m1,m1/m2}
          \draw[grad] (\from) -- node[midway, anchor=west] {$\grad$} (\to);

          \draw[grad] (m2) -- node[midway, fill=white]{$\Large\cdots$} node[near start, anchor=west] {$\grad$}  (mn);
        \end{scope}

      \end{scope}
    \end{tikzpicture}
  }
  \caption{Simple example illustrating the difference of the Surrogate Loss
  (SL) approach to \magicbox{}. Stochastic nodes are depicted in orange, costs in
  gray, surrogate losses in blue, \magicbox{} in purple, and gradient estimators in red. Note that for
  second-order gradients, SL requires the construction of an intermediate
  stochastic computation graph and due to taking a sample of the cost $\hat
  g_\SL$, the dependency on $\theta$ is lost, leading to an incorrect
  second-order gradient estimator. Arrows from $\theta,x$ and $f$ to gradient estimators omitted for clarity.}
  \label{fig:bug}
\end{figure}
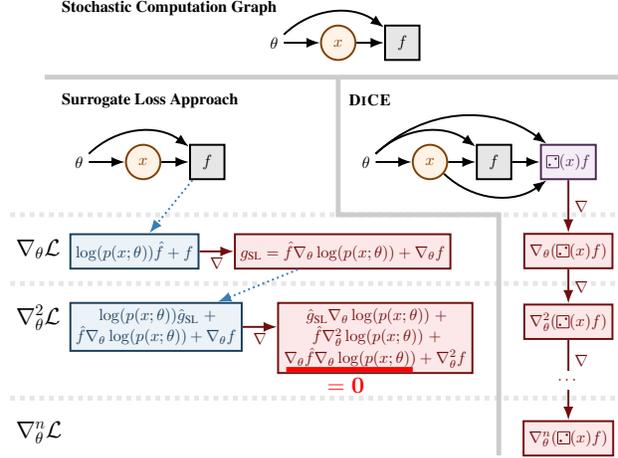

\label{sec:slbug}
While \citet{schulman2015gradient} focus on first order gradients, they state 
that a recursive application of SL can generate higher order gradient 
estimators.
However, as we demonstrate in this section, because the SL approach treats part of the objective as a sampled cost, the corresponding terms lose a functional dependency on the sampling distribution.
This leads to missing terms in the estimators of higher order gradients.

Consider the following example, where a single parameter $\theta$ defines a sampling distribution $p(x;\theta)$ and the objective is $f(x,\theta)$.
\begin{align}
\SL(\obj) &=  \log p(x;\theta) \hat{f}(x) + f(x;\theta) \nonumber \\
 (\gradtheta \obj)_{\SL} &= \E_x [ \gradtheta \SL(\obj)  ] \nonumber \\
&=\E_x [\hat{f}(x)  \gradtheta \log p(x;\theta) + \gradtheta f(x;\theta) ]
\label{eq:sl_first_order} \\
&=\E_x [ g_\SL(x;\theta) ].  \nonumber
\end{align}
The corresponding SCG is depicted at the top of Figure~\ref{fig:bug}.
Comparing \eqref{eq:exact_gradient} and \eqref{eq:sl_first_order}, note that the first term, $\hat{f}(x)$ has lost its functional dependency on $\theta $, as indicated by the hat notation and the lack of a $\theta$ argument.
While these terms evaluate to the same estimate of the first order gradient, 
the lack of the dependency yields a discrepancy between the exact derivation of 
the second order gradient and a second application of SL:
\begin{align}
\SL(g_\SL(x;\theta) ) &= \log p(x;\theta) \hat{g}_\SL(x) + g_\SL(x;\theta) \nonumber \\
(\gradtheta^2 \obj)_\SL &=  \E_x [ \gradtheta \SL( g_{\SL}) ] \nonumber \\
&=\E_x [\hat{g}_\SL(x)  \gradtheta \log p(x;\theta) + \gradtheta
g_\SL(x;\theta) ].  \label{eq:sl_2nd_order}
\end{align}
By contrast, the exact derivation of $\gradtheta^2 \obj$ results in the following expression:
\begin{align}
\gradtheta^2 \obj &=  \gradtheta \E_x[ g(x;\theta)] \nonumber \\
&= \E_x [g(x;\theta)  \gradtheta \log p(x;\theta) + \gradtheta g(x;\theta) ].
\label{eq:exact_2nd_order}
\end{align}
Since $g_{\SL}(x;\theta)$ differs from $g(x;\theta)$ only in its dependencies 
on $\theta$, $g_{\SL}$ and $g$ are identical when \emph{evaluated}.
However, due to the missing dependencies in $g_\SL$, the \emph{gradients} w.r.t.\ $\theta$, which appear in the higher order gradient estimates in  \eqref{eq:sl_2nd_order} and \eqref{eq:exact_2nd_order}, differ:
\begin{align*}
\gradtheta g(x;\theta) &= \gradtheta f(x; \theta) \gradtheta \log(p(x; \theta)) \\
&\qquad + f(x; \theta) \gradtheta^2 \log(p(x; \theta)) \\
&\qquad + \gradtheta^2 f(x; \theta), \\
\gradtheta g_\SL(x;\theta) &= \hat{f}(x) \gradtheta^2 \log(p(x; \theta)) \\
&\qquad + \gradtheta^2 f(x; \theta).
\end{align*}
We lose the term $\gradtheta f(x; \theta) \gradtheta \log(p(x;\theta))$ in the
second order SL gradient because $\gradtheta \hat{f}(x) = 0$ (see left part of Figure~\ref{fig:bug}). This issue occurs
immediately in the second order gradients when $f$ depends directly on
$\theta$. However, as $g(x; \theta)$ always depends on $\theta$, the SL
approach always fails to produce correct third or higher order gradient
estimates even if $f$ depends only indirectly on $\theta$.

\subsection{Example}
Here is a toy example to illustrate a possible failure case. Let $x \sim
\bernoulli(\theta)$ and  $f(x, \theta) =   x (1-\theta) + (1-x) (1 + \theta)$.
For this simple example we can exactly evaluate all terms:
\begin{align*}
\obj &= \theta  (1-\theta) + (1-\theta)(1 + \theta) \\
\gradtheta \obj &= -4 \theta + 1 \\
\gradtheta^2 \obj &= -4 \\
 \end{align*}
Evaluating the expectations for the SL gradient estimators analytically results 
in the following terms, with an incorrect second-order estimate:
 \begin{align*}
(\gradtheta \obj)_{\SL} &= -4 \theta + 1 \\
(\gradtheta^2 \obj)_{\SL} &= -2 \\
 \end{align*}
If, for example, the Newton-Raphson method was used to optimise $\obj$, the  
solution could be found in a single iteration with the correct 
Hessian.
In contrast, the wrong estimates from the SL approach would require damping to 
approach the optimum at all, and many more iterations would be needed.

The failure mode seen in this toy example appears whenever the objective 
includes a regularisation term that depends on $\theta$, and is also impacted 
by the stochastic samples.
One example in a practical algorithm is soft $Q$-learning for RL 
\citep{schulman2017equivalence}, which regularises the policy by adding an 
entropy penalty to the rewards.
This penalty encourages the agent to maintain an exploratory policy, reducing 
the probability of getting stuck in local optima.
Clearly the penalty depends on the policy parameters $\theta$.
However, the policy entropy also depends on the states visited, which in 
turn depend on the stochastically sampled actions.
As a result, the entropy regularised RL objective in this algorithm has
the exact property leading to the failure of the SL approach shown above.
Unlike our toy analytic example, the consequent errors do not just appear
as a rescaling of the proper higher order gradients, but depend in a 
complex way on the parameters $\theta$.
Any second order methods with such a regularised objective therefore 
requires an alternate strategy for generating gradient estimators, even 
setting aside the awkwardness of repeatedly generating new surrogate objectives.

\section{Correct Gradient Estimators with DiCE}
\label{sec:MagicBox}
In this section, we propose the \emph{Infinitely \textbf{Di}fferentiable Monte-\textbf{C}arlo \textbf{E}stimator} (\magicbox), a practical algorithm for programatically generating correct gradients of any order in arbitrary SCGs.
The naive option is to recursively apply the update rules in \eqref{eq:exact_gradient} that map from $f(x;\theta)$ to the estimator of its derivative $g(x;\theta)$.
However, this approach has two deficiencies. First, by defining gradients directly, it fails to provide an objective that can be used in standard deep learning libraries.
Second, these naive gradient estimators violate the auto-diff paradigm for 
generating further estimators by repeated differentiation since in 
general $\gradtheta f(x;\theta) \neq g(x; \theta)$.
Our approach addresses these issues, as well as fixing the missing 
terms from the SL approach.

As before, $\obj = \E [ \sum_{c \in \costs} c ]$ is the objective in an SCG.
The correct expression for a gradient estimator that preserves all required dependencies for further differentiation is:
\begin{align}
\gradtheta \obj
& = \E\Bigg[\sum_{c\,\in\,\costs} \Bigg( c\sum_{w \in \wc}
\gradtheta\log p(w\ |\ \deps_w) \nonumber \\
& \qquad\qquad\quad  +  \gradtheta c(\deps_c)\Bigg) \Bigg],
\label{eq:js_derivative}
\end{align}
where $ \wc = \{w\ |\ w \in \stochastic, w \influences c,  \theta \influences
w\} $, \ie{}, the set of stochastic nodes that depend on $\theta$ and influence 
the cost $c$.
For brevity, from here on we suppress the $\deps$ notation, assuming all
probabilities and costs are conditioned on their relevant ancestors in the SCG.

Note that \eqref{eq:js_derivative} is the generalisation  of
\eqref{eq:exact_gradient} to arbitrary SCGs. The proof is given by \citet[Lines 1-10, Appendix
A]{schulman2015gradient}.
Crucially, in Line 11 the authors then replace $c$ by $\hat{c}$, severing the
dependencies required for correct higher order gradient estimators.
As described in Section~\ref{ssec:sl}, this was done so that the SL approach
reproduces the score function estimator after a single differentiation
and can thus be used as an objective for backpropagation in a deep learning
library.

To support correct higher order gradient estimators, we propose \magicbox{},  which relies heavily on a novel operator, \magicname{} ($\magic$). \magicname{} takes a set of stochastic nodes $\W$ as input and has the following two properties by design:
\begin{enumerate}
    \item $\magic (\W) \eval  1$,
    \item $\gradtheta \magic (\W) =  \magic (\W) \sum_{w \in \W}  \gradtheta
    \log(p(w;\theta))$.
\end{enumerate}
Here, $  \eval $ indicates ``\emph{evaluates to}'' in contrast to full equality,
$=$, which includes equality of all gradients. In the auto-diff paradigm,
$\eval$ corresponds to a forward pass evaluation of a term.
Meanwhile, the behaviour under differentiation in property (2) indicates the
new graph nodes that will be constructed to hold the gradients of that object.
Note that that $\magic (\W) $ reproduces the dependency of the gradient on the sampling distribution under differentiation through the requirements above.
Using $\magic$, we can next define the \magicbox{} objective, $\obj_{\smagic}$:
\begin{equation}
\label{eq:magic_obj}
\obj_{\smagic} = \sum_{c \in \costs}  \magic ( \wc ) c.
\end{equation}
Below we prove that the \magicbox{} objective indeed produces correct arbitrary order gradient estimators under differentiation.

\begin{theorem}
$\E [ \gradtheta^n \obj_{\smagic} ]  \eval \gradtheta^n  \obj, \forall n \in
\{0,1,2, \dots\}$.
\end{theorem}
\begin{proof}
For each cost node $c \in \costs$, we define a sequence of nodes, $c^n,  n \in \{0, 1, \dots\}$ as follows:
\begin{align}
c^0  & = c,  \nonumber \\
\E [ c^{n+1} ]  & =  \gradtheta \E [ c^n ]. \label{eq:recursion_req}
\end{align}
By induction it follows that $\E [c^n ] =\gradtheta^n \E [ c ] \; \forall n$, 
\ie{}, $c^n$ is an estimator of the $n$th order derivative of the objective 
$\E [c]$.

We further define  $c^n_{\smagic} = c^n \magic (\wcn)$.
Since $\magic (x) \eval 1$, clearly $c^n_{\smagic} \eval c^n$.
Therefore $\E [ c^n_{\smagic} ] \eval \E [c^n  ] = \gradtheta^n \E [c]$, \ie{}, 
$c^n_{\smagic}$ is also a valid estimator of the $n$th order derivative of the 
objective.
Next, we show that $ c^n_{\smagic}$ can be generated by differentiating $c^0_{\smagic}$ $n$ times.
This follows by induction, if $\gradtheta c^n_{\smagic}  = c^{n+1}_{\smagic}$, which we prove as follows:
\begin{align}
  \gradtheta c^n_{\smagic}  &=   \gradtheta ( c^n \magic ( \wcn)  ) \nonumber \\
                            &= c^n  \gradtheta \magic (  \wcn)  +  \magic (  \wcn )  \gradtheta c^n  \nonumber  \\
                            &= c^n \magic (  \wcn)  \left(\sum_{w \in \wcn}  \gradtheta \log(p(w;\theta))\right)\nonumber\\
                            &\quad + \magic (  \wcn )  \gradtheta c^n  \nonumber  \\
                            &= \magic (\wcn) \left( \gradtheta c^n  +  c^n  \sum_{w
                              \in \wcn}  \gradtheta \log(p(w;\theta))  \right)
                              \label{eq:box_recursion_penultimate} \\
                            &= \magic ({ \mathcal{W}_{c^{n+1}}})  c^{n+1}  =
                              c^{n+1}_{\smagic} . \label{eq:box_recursion_final}
\end{align}
To proceed from \eqref{eq:box_recursion_penultimate} to
\eqref{eq:box_recursion_final}, we need two additional steps.
First, we require an expression for $c^{n+1}$.
Substituting $\obj = \E [c^n]$ into \eqref{eq:js_derivative} and comparing to
\eqref{eq:recursion_req}, we find the following map from $c^n$ to $c^{n+1}$:
\begin{equation}
\label{eq:c_recursion}
c^{n+1} =  \gradtheta c^n + c^n \sum_{w \in \wcn} \gradtheta \log p(w;\theta).
\end{equation}

The term inside the brackets in \eqref{eq:box_recursion_penultimate} is
identical to $c^{n+1}$.
Secondly, note that \eqref{eq:c_recursion} shows that $c^{n+1}$ depends only on
$c^{n}$ and $\wcn$.
Therefore, the stochastic nodes which influence $c^{n+1}$ are the same as those
which influence $c^n$. So $\wcn = \mathcal{W}_{c^{n+1}}$, and we arrive at
\eqref{eq:box_recursion_final}.

To conclude the proof, recall that $c^n$ is the estimator for the $n$th
derivative of $c$, and that $c^n_{\smagic} \eval c^n$. Summing over $c \in \costs$
then gives the desired result.
\end{proof}

\textbf{Implementation of \magicbox.}  \magicbox{} is easy to implement in 
standard deep learning libraries \footnote{A previous version of 
tf.contrib.bayesflow authored by Josh Dillon also used this implementation 
trick. 
}:
\begin{align*}
 \magic (\W) &= \exp\big(\tau - \stopgrad( \tau )\big), \\
\tau  &= \sum_{w \in \W} { \log(p(w;\theta))},
\end{align*}
where $\stopgrad$ is an operator that sets the gradient of the
operand to zero, so $\grad_x \text{\stopgrad}(x) = 0$.\footnote{This operator exists in PyTorch as \verb~detach~ and in TensorFlow as \verb~stop\_gradient~.}

Since $\stopgrad(x) \eval x$, clearly $ \magic (\W) \eval 1$.
Furthermore:
\begin{align*}
\gradtheta \magic (\W) &=
\gradtheta \exp\big(\tau - \stopgrad( \tau )\big)  \\
&= \exp\big(\tau - \stopgrad( \tau )\big) \gradtheta (\tau - \stopgrad( \tau )) \\
&= \magic (\W) ( \gradtheta \tau + 0 )\\
&= \magic (\W) \sum_{w \in \W} \gradtheta \log(p(w;\theta)).
\end{align*}

With this implementation of the $\magic$-operator, it is now straightforward to
construct $\obj_{\smagic}$ as defined in \eqref{eq:magic_obj}. This procedure
is demonstrated in Figure~\ref{fig:overview}, which shows a reinforcement
learning  use case. In this example, the cost nodes are rewards that depend on
stochastic actions, and the total objective is $J = \E [\sum r_t ]$. We
construct a \magicbox{} objective $J_{\smagic } = \sum_t \magic (\{a_{t'}, t'
\leq t\})r_t$.
Now $\E [ J_{\smagic } ] \eval J$ and $\E [ \gradtheta^n  J_{\smagic } ]\eval
\gradtheta^n J$, so  $J_{\smagic }$ can both be used to estimate the return and
to produce estimators for any order gradients under auto-diff, which can be used for higher order methods.

Note that \magicbox{} can be equivalently expressed with $\magic (\W) = 
\tilde{p}
/\stopgrad( \tilde{p} ), \tilde{p} = \prod_{w \in \W} {p(w;\theta)}$. We use the 
exponentiated form to emphasise the 
generator-like functionality of the operator and to ensure numerical stability.

\textbf{Causality.}
The SL approach handles causality by summing over stochastic
nodes, $w$, and multiplying $\grad \log(p(w))$ for each stochastic node
with a sum of the \emph{downstream} costs, $\hat{Q}_w$.
In contrast, the \magicbox{} objective sums over costs, $c$, and
multiplies each cost with a sum over the gradients of log-probabilities
from \emph{upstream} stochastic nodes, $\sum_{w \in \wc}\grad \log (p(w))$.

In both cases, integrating causality into the gradient estimator leads to
reduction of variance compared to the naive approach of multiplying the full
sum over costs with the full sum over grad-log-probabilities.

However, the second formulation leads to greatly reduced conceptual complexity
when calculating higher order terms, which we exploit in the definition of the
\magicbox{} objective. This is because each further gradient estimator
maintains the same backward looking dependencies for each term in the original
sum over costs, \ie{}, $\wcn = \mathcal{W}_{c^{n+1}}$. 

In contrast, the SL
approach is centred around the stochastic nodes, which each become associated 
with
a growing number of downstream costs after each differentiation.
Consequently, we believe that our \magicbox{} objective is
more intuitive, as it is conceptually centred around the
original objective and remains so under repeated differentiation.

\textbf{Variance Reduction.} We can include a baseline term in the definition 
of the \magicbox{} objective:
\begin{equation}
\label{eq:magic_obj}
\obj_{\smagic} = \sum_{c \in \costs}  \magic( \wc ) c + \sum_{w \in
\stochastic}(
1 - \magic(\{w\}))b_w.
\end{equation}
The baseline $b_w$ is a design choice and can be any function of nodes not
influenced by $w$. As long as this condition is met, the baseline does not
change the expectation of the gradient estimates, but can considerably reduce
the variance. A common choice is the average cost.

Since $( 1 - \magic(\{w\})) \eval 0$, this implementation of the baseline 
leaves the evaluation of the estimator $\obj_{\smagic}$ of the original 
objective unchanged.

\textbf{Hessian-Vector Product.} The Hessian-vector, $v^\top  H$, is useful for a number of algorithms, such as estimation of eigenvectors and eigenvalues of $H$~\citep{pearlmutter1994fast}. 
Using \magicbox, $v^\top  H$ can be
implemented efficiently without having to compute the full Hessian.
Assuming $v$ does
not depend on $\theta$ and using $^\top$ to indicate the transpose:
\begin{align*}
v ^\top H &= v^\top \grad^2  \obj_{\smagic} \\
     &= v^\top ( \grad^\top  \grad  \obj_{\smagic} )\\
      &= \grad^\top  (v^\top  \grad \obj_{\smagic} ).
\end{align*}
In particular, $(v^\top  \grad \obj_{\smagic} )$ is a scalar, making this
implementation well suited for auto-diff.

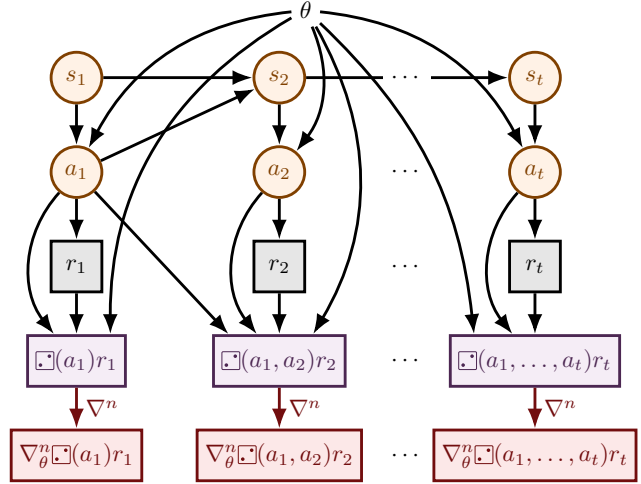
\begin{figure}[t!]
  \centering
      \scalebox{0.9}{
      \begin{tikzpicture}
        \node[stochastic] (s1) at (0,0) {$s_1$};
        \node[stochastic, anchor=west] (s2) at ($(s1)+(\xdist,0)$) {$s_2$};
        \node[stochastic, anchor=west] (st) at ($(s2)+(1.3*\xdist,0)$) {$s_t$};
        \draw[dep] (s2) -- node[midway, fill=white] {$\cdots$} (st);

        \node[stochastic, anchor=north] (a1) at ($(s1)+(0,-\ydist)$) {$a_1$};
        \node[stochastic, anchor=north] (a2) at ($(s2)+(0,-\ydist)$) {$a_2$};
        \node[stochastic, anchor=north] (at) at ($(st)+(0,-\ydist)$) {$a_t$};
        \node[] (ad) at ($(a2)!0.5!(at)$) {$\cdots$};

        \node[cost, anchor=north] (r1) at ($(a1)+(0,-\ydist)$) {$r_1$};
        \node[cost, anchor=north] (r2) at ($(a2)+(0,-\ydist)$) {$r_2$};
        \node[cost, anchor=north] (rt) at ($(at)+(0,-\ydist)$) {$r_t$};
        \node[] (rd) at ($(r2)!0.5!(rt)$) {$\cdots$};

        \foreach \from/\to in {s1/s2,s1/a1,s2/a2,st/at,a1/r1,a2/r2,at/rt,a1/s2}
        \draw[dep] (\from) -- (\to);

        \node[input] (t) at ($(s1)!0.5!(st)+(0,1)$) {$\theta$};
        \path[dep] (t) edge[bend right=28] (a1);
        \path[dep] (t) edge[bend left] (a2);
        \path[dep] (t) edge[bend left=28] (at);

        \node[magicbox, anchor=north] (m1) at ($(r1)+(0,-\ydist)$) {$\magic (a_1)r_1$};
        \node[magicbox, anchor=north] (m2) at ($(r2)+(0,-\ydist)$) {$\magic (a_1,a_2)r_2$};
        \node[magicbox, anchor=north] (mt) at ($(rt)+(0,-\ydist)$) {$\magic (a_1,\ldots,a_t)r_t$};
        \node[] (md) at ($(m2)!0.5!(mt)$) {$\cdots$};

        \node[gradbox, anchor=north] (dm1) at ($(m1)+(0,-\ydist)$) {$\grad^n_\theta \magic (a_1)r_1$};
        \node[gradbox, anchor=north] (dm2) at ($(m2)+(0,-\ydist)$) {$\grad^n_\theta \magic (a_1,a_2)r_2$};
        \node[gradbox, anchor=north] (dmt) at ($(mt)+(0,-\ydist)$) {$\grad^n_\theta \magic (a_1,\ldots,a_t)r_t$};
        \node[] (md) at ($(dm2)!0.5!(dmt)$) {$\cdots$};

        \foreach \from/\to in {r1/m1,r2/m2,rt/mt}
        \draw[dep] (\from) -- (\to);

        \foreach \from/\to in {m1/dm1,m2/dm2,mt/dmt}
        \draw[grad] (\from) -- node[midway, anchor=west] {$\grad^n$} (\to);

        \draw[dep] (a1) -- ($(m2.north)+(-0.75,0)$);
        \foreach \from/\to in {a1/m1,a2/m2,at/mt}
        \path[dep] (\from) edge[bend right=40] (\to);

        \path[dep] (t) edge[bend right=30] ($(m1.north)+(0.5,0)$);
        \path[dep] (t) edge[bend left] ($(m2.north)+(0.5,0)$);
        \path[dep] (t) edge[bend left=20] ($(mt.north)-(0.9,0)$);
      \end{tikzpicture}
    }

  \caption{\magicbox{} applied to a reinforcement learning problem. A stochastic policy conditioned on $s_t$ and $\theta$ produces actions, $a_t$, which lead to rewards $r_t$ and next states, $s_{t+1}$.  Associated with each reward is a \magicbox{} objective that takes as input the set of all causal dependencies that are functions of $\theta$, \ie, the actions. Arrows from $\theta,a_i$ and $r_i$ to gradient estimators omitted for clarity.}
  \label{fig:overview}
\end{figure}

\section{Case Studies}

While the main contribution of this paper is to provide a novel general
approach for any order gradient estimation in SCGs, we also provide a
proof-of-concept empirical evaluation for a set of case studies, carried out on
the \emph{iterated prisoner's dilemma} (IPD). In IPD, two agents iteratively
play matrix games with two possible actions: (C)ooperate and (D)efect. The
possible outcomes of each game are DD, DC, CD, CC with the corresponding first
agent payoffs, -2, 0, -3, -1, respectively.
This setting is useful because (1) it has a nontrivial but analytically
calculable value function, allowing for verification of gradient estimates, and
(2) differentiating through the learning steps of other agents in multi-agent
RL is a highly relevant application of higher order policy gradient estimators 
in RL \citep{foerster2018}.

\begin{figure}[t]
    \includegraphics[width =0.92\linewidth]{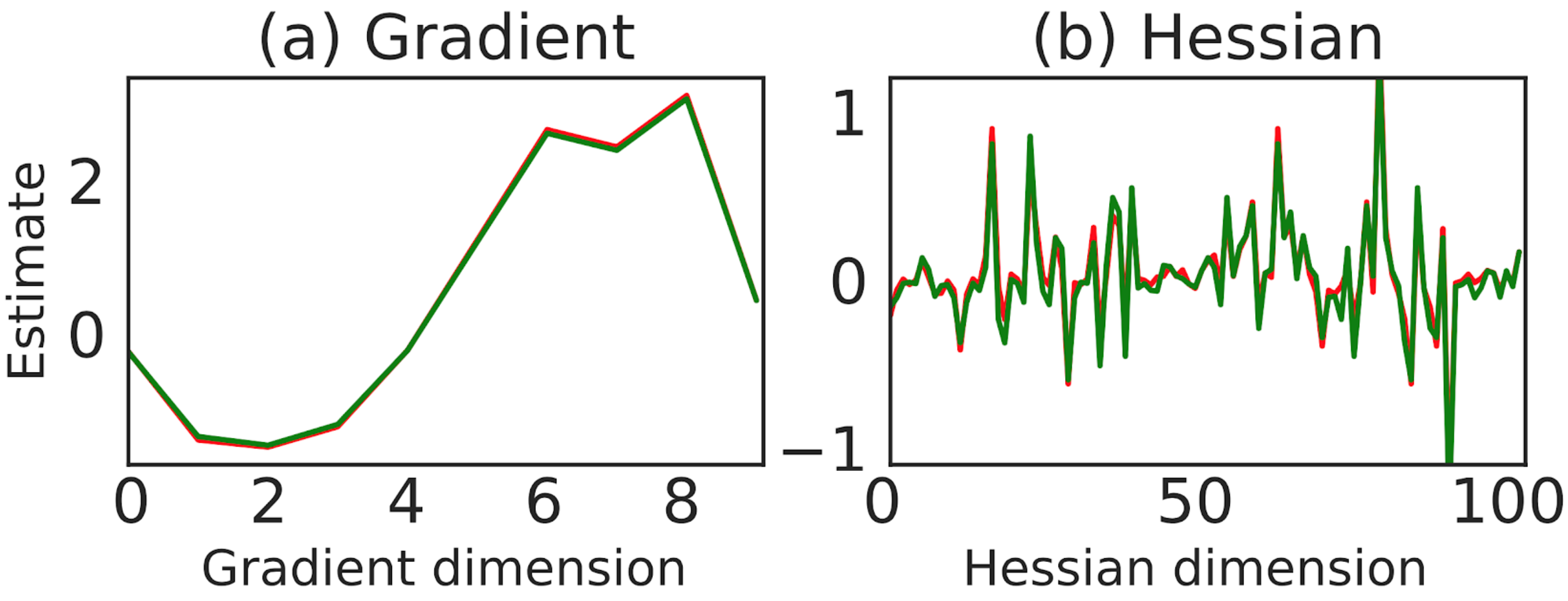}
    \caption{For the iterated prisoner's dilemma,
    shown is the flattened true (red) and estimated (green) gradient (left) and Hessian (right) using the
    first and second derivative of \magicbox{} and the exact value function respectively.  The correlation coefficients are  0.999 for the gradients and 0.97 for the Hessian;
    the sample size is 100k.}
\label{fig:hessian}
	\vspace{-0em}
\end{figure}

\textbf{Empirical Verification.} We first verify that \magicbox{} recovers 
gradients and Hessians in stochastic computation graphs.
To do so, we use \magicbox{} to estimate gradients and Hessians of the expected
return for fixed policies in IPD.

As shown in Figure~\ref{fig:hessian}, we find that indeed  the \magicbox{} estimator matches both the gradients  (a) and the Hessians (b) for both agents accurately.  
Furthermore, Figure~\ref{fig:error_by_learnig} shows how the estimate of the gradient improve as the value function becomes more accurate during training, in (a). Also shown is the quality of the gradient estimation as a function of sample size with and without a baseline, in (b). 
Both plots show that the baseline is a key component of \magicbox~ for accurate 
estimation of gradients. 

\textbf{\magicbox{} for multi-agent RL.} In \emph{learning with opponent-learning awareness} (LOLA), \citet{foerster2018} show that agents that differentiate through the learning step of their opponent converge to  Nash equilibria with higher social welfare in the IPD.

Since the standard policy gradient learning step for one agent has no dependency on the parameters of the other agent (which it treats as part of the environment), LOLA relies on a Taylor expansion of the expected return in combination with an analytical derivation of the second order gradients to be able to differentiate through the expected return after the opponent's learning step.

Here, we take a more  direct approach, made possible by \magicbox{}.
Let $\pi_{\theta_1}$ be the policy of the LOLA agent and let $\pi_{\theta_2}$ be the policy of its opponent and vice versa.
Assuming that the opponent learns using policy gradients, LOLA-\magicbox{} agents learn by directly optimising the following stochastic objective w.r.t. $\theta_1$:
\begin{equation}
    \label{eq:meta-lola-obj}
    \begin{split}
        \obj^1(\theta_1, \theta_2)_{\text{LOLA}} &= \E_{\pi_{\theta_1}, \pi_{\theta_2 + \Delta \theta_2(\theta_1, \theta_2)}}\left[\obj^1\right], \text{where} \\
        \Delta \theta_2(\theta_1, \theta_2) &= \alpha_2 \nabla_{\theta_2} \E_{\pi_{\theta_1}, \pi_{\theta_2}} \left[\obj^2\right],
    \end{split}
\end{equation}
where $\alpha_2$ is a scalar step size and $\obj^i = \sum_{t=0}^T \gamma^{t} r^i_t$ is the sum of discounted returns for agent $i$.

To evaluate these terms directly, our variant of LOLA unrolls the learning process of the opponent, which is functionally similar to model-agnostic meta-learning~\citep[MAML,][]{finn2017}.
In the MAML formulation, the gradient update of the opponent, $ \Delta 
\theta_2$, corresponds to the inner loop (typically the training objective) and the 
gradient update of the agent itself to the outer loop (typically the test 
objective).
Algorithm~\ref{alg:lola-dice} describes the procedure we use to compute updates for the agent's parameters.

\begin{figure}[t]
    \includegraphics[width =0.92\linewidth]{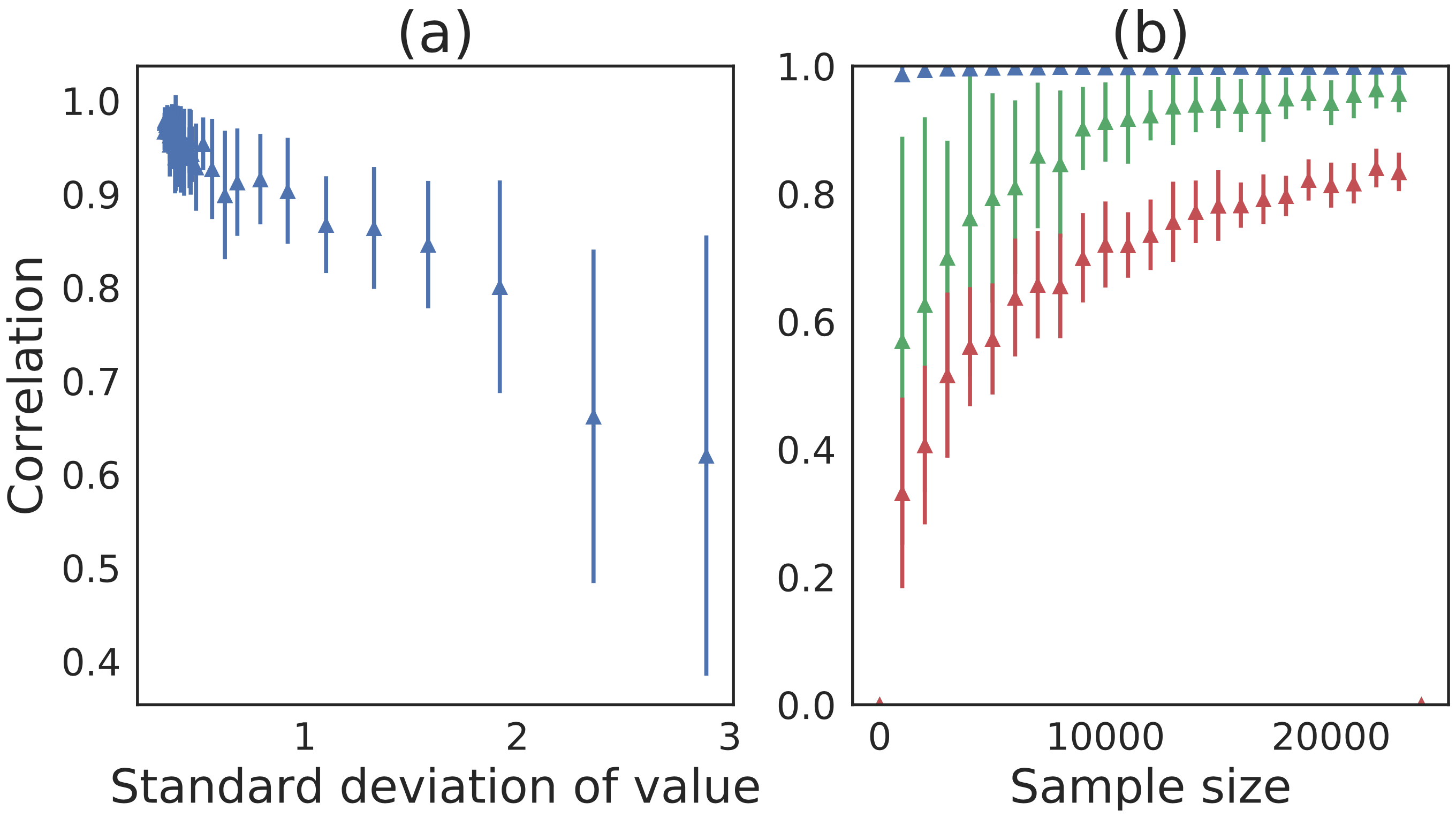}
    \caption{Shown in (a) is the correlation of the gradient estimator (averaged across agents) as a function of the estimation error of the baseline when using a sample size of 128 and in (b) as a function of sample size when using a converged baseline (in blue) and no baseline (in green) for gradients and in red for Hessian.  Errors bars indicate standard deviation on both plots.}
\label{fig:error_by_learnig}
	\vspace{-0em}
\end{figure}

Using the following \magicbox{}-objective to estimate gradient steps for agent $i$, we are able to preserve all dependencies:
\begin{equation}
    \obj^i_{\smagic(\theta_1, \theta_2)} = \sum_{t} \magic\left(\left\{a_{j \in \{1,2\}}^{t^{\prime} \leq t}\right\}\right) \gamma^t r^i_t,
\end{equation}
where $\left\{a_{j\in \{1, 2\}}^{t^{\prime} \leq t}\right\}$ is the set of all actions taken by both agents up to time $t$.
To save computation, we cache the $\Delta \theta_i$ of the inner loop when unrolling the outer loop policies in order to avoid recalculating them at every time step.

\begin{algorithm}[H]
    \caption{LOLA-DiCE: policy gradient update for $\theta_1$}\label{alg:lola-dice}
    \small
    \begin{algorithmic}[1]
        \INPUT Policy parameters of the agent, $\theta_1$, and of the opponent, $\theta_2$
        \STATE Initialize: $\theta_2^\prime \leftarrow \theta_2$
        \FOR[inner loop lookahead steps]{$k$ in $1 \dots K$}
            \STATE Rollout trajectories $\tau_k$ under $(\pi_{\theta_1}, \pi_{\theta_2^\prime})$
            \STATE Update: $\theta_2^\prime \leftarrow \theta_2^\prime + \alpha_2 \nabla_{\theta_2^\prime} \obj^2_{\smagic(\theta_1, \theta_2^\prime)}$ \COMMENT{lookahead update}
        \ENDFOR
        \STATE Rollout trajectories $\tau$ under $(\pi_{\theta_1}, \pi_{\theta_2^\prime})$.
        \STATE Update: $\theta_1^\prime \leftarrow \theta_1 + \alpha_1 \nabla_{\theta_1} \obj^1_{\smagic(\theta_1, \theta_2^\prime)}$ \COMMENT{PG update}
        \OUTPUT $\theta_1^\prime$.
    \end{algorithmic}
\end{algorithm}

Using \magicbox{}, differentiating through $\Delta \theta_2$ produces the correct higher order gradients, which is critical for LOLA.
By contrast, simply differentiating through the SL-based first order gradient estimator multiple times, as was done for MAML \citep{finn2017}, results in omitted gradient terms and a biased gradient estimator, as pointed out by \citet{al2017continuous} and \citet{stadie2018some}.

Figure~\ref{fig:lola_maml} shows a comparison between the LOLA-\magicbox{} agents and the original formulation of LOLA.
In our experiments, we use a time horizon of 150 steps and a reduced batch size of 64; the lookahead gradient step, $\alpha$, is set to 1 and the learning rate is 0.3.
Importantly, given the approximation used, the LOLA method was restricted to a single step of opponent learning.
In contrast, using \magicbox~ we can unroll and differentiate through an arbitrary number of the opponent learning steps.

The original LOLA implemented via second order gradient corrections shows no 
stable learning, as it requires much larger batch sizes ($\sim$$4000$).
By contrast, LOLA-\magicbox{} agents discover strategies of high social welfare, replicating the results of the original LOLA paper in a way that is both more direct, efficient and establishes a common formulation between MAML and LOLA.

\begin{figure}[t!]
	\centering
	\includegraphics[width=\linewidth]{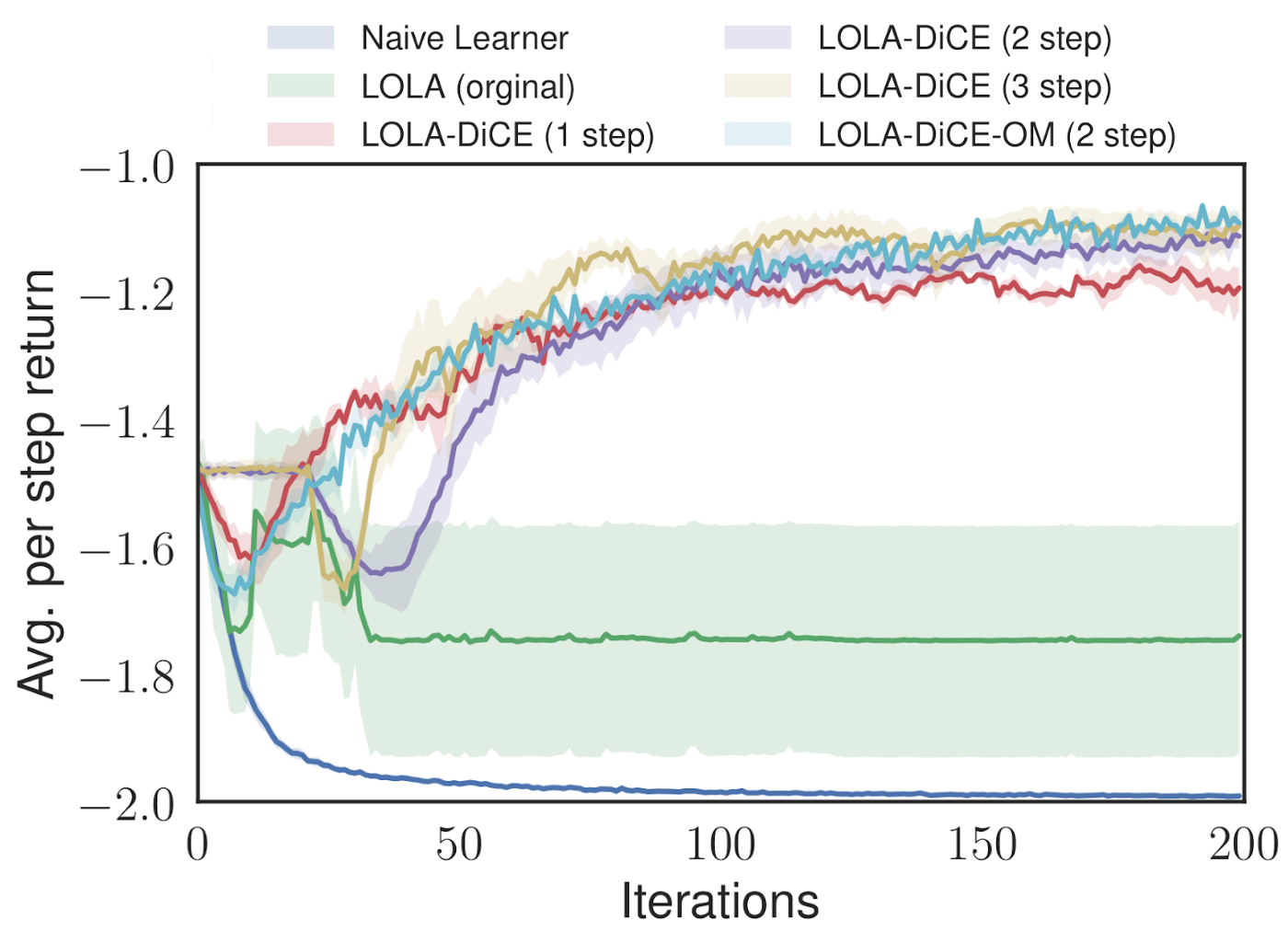}
	\caption{Joint average per step returns for different training methods. Comparing Naive Learning with the original LOLA algorithm and LOLA-DiCE with a varying number of look-ahead steps. 
		Shaded areas represent the 95\% confidence intervals based on 5 runs.
		All agents used batches of size 64, which is more than 60 times smaller 
		than the size required in the original LOLA paper.}
	\label{fig:lola_maml}
	\vspace{-0.0em}
\end{figure}
\section{Related Work}

Gradient estimation is well studied, although many methods have
been named and explored independently in different fields, and the primary
focus has been on first order gradients.
\citet{fu2006gradient} provides an overview of methods from the point of view
of simulation optimization.

The score function (SF) estimator, also referred to as the likelihood ratio
estimator or REINFORCE, has received considerable attention in many fields. In
reinforcement learning, policy gradient methods \citep{williams1992simple} have
proven highly successful, especially when combined with variance reduction
techniques \citep{weaver2001optimal,grondman2012survey}.
The SF estimator has also been used in the analysis of stochastic systems
\citep{glynn1990likelihood}, as well as for variational inference
\citep{wingate2013automated,ranganath2014black}.
\citet{kingma2013auto} and \citet{rezende2014stochastic} discuss Monte-Carlo
gradient estimates in the case where the stochastic parts of a model can be reparameterised.

These approaches are formalised for arbitrary computation graphs by 
\citet{schulman2015gradient}, but to our knowledge our paper is the first to 
present a practical and correct approach for generating higher order gradient 
estimators utilising auto-diff. To easily make use of these estimates for 
optimising neural network models,
automatic differentiation for backpropagation has been widely used
\citep{baydin2015automatic}.

One rapidly growing application area for such higher order gradient estimates is meta-learning for reinforcement learning.
\citet{finn2017} compute a loss after a number of policy gradient learning steps, differentiating through the learning step to find parameters that can be quickly fine-tuned for different tasks.
\citet{li2017meta} extend this work to also meta-learn the fine-tuning step 
direction and magnitude.
\citet{al2017continuous} and \citet{stadie2018some} derive the proper higher order gradient estimators for their work by reapplying the score function trick.
\citet{foerster2018} use a multi-agent version of the same higher order gradient estimators in combination with a Taylor expansion of the expected return. None present a general strategy for constructing higher order gradient estimators for arbitrary stochastic computation graphs.

\section{Conclusion}
We presented \magicbox{}, a general method for computing any order gradient 
estimators for stochastic computation graphs. \magicbox{} resolves the 
deficiencies of current approaches for computing higher order gradient 
estimators: analytical calculation is error-prone and incompatible with 
auto-diff, while repeated application of the surrogate loss approach is 
cumbersome and, as we show, leads to incorrect estimators in many cases. We 
prove the 
correctness of \magicbox{} estimators, introduce a simple practical 
implementation of \magicbox{} for use in deep learning frameworks, and validate 
its correctness and utility in a multi-agent reinforcement learning problem.
We believe \magicbox{} will unlock further exploration and adoption of higher order learning methods in meta-learning, reinforcement learning, and other applications of stochastic computation graphs.
As a next step we will extend and improve the variance reduction of \magicbox{} in order to provide a simple end-to-end solution for higher order gradient estimation. In particular we hope to include solutions such as REBAR~\citep{tucker2017rebar} in the \magicbox{} operator.   
\section*{Acknowledgements}
This project has received funding from the European Research Council (ERC) under the European Union's Horizon 2020 research and innovation programme (grant agreement number 637713) and National Institute of Health (NIH R01GM114311).
It was also supported by the Oxford Google DeepMind Graduate Scholarship and the UK EPSRC CDT in Autonomous Intelligent Machines and Systems.
We would like to thank Misha Denil, Brendan Shillingford and Wendelin  Boehmer for providing feedback on the manuscript.

\bibliography{ref}

\begin{thebibliography}{23}
\providecommand{\natexlab}[1]{#1}
\providecommand{\url}[1]{\texttt{#1}}
\expandafter\ifx\csname urlstyle\endcsname\relax
  \providecommand{\doi}[1]{doi: #1}\else
  \providecommand{\doi}{doi: \begingroup \urlstyle{rm}\Url}\fi

\bibitem[Abadi et~al.(2016)Abadi, Barham, Chen, Chen, Davis, Dean, Devin,
  Ghemawat, Irving, Isard, Kudlur, Levenberg, Monga, Moore, Murray, Steiner,
  Tucker, Vasudevan, Warden, Wicke, Yu, and Zheng]{abadi2016tensorflow}
Abadi, M., Barham, P., Chen, J., Chen, Z., Davis, A., Dean, J., Devin, M.,
  Ghemawat, S., Irving, G., Isard, M., Kudlur, M., Levenberg, J., Monga, R.,
  Moore, S., Murray, D.~G., Steiner, B., Tucker, P.~A., Vasudevan, V., Warden,
  P., Wicke, M., Yu, Y., and Zheng, X.
\newblock Tensorflow: {A} system for large-scale machine learning.
\newblock In \emph{12th {USENIX} Symposium on Operating Systems Design and
  Implementation, {OSDI} 2016, Savannah, GA, USA, November 2-4, 2016.}, pp.\
  265--283, 2016.

\bibitem[Al{-}Shedivat et~al.(2017)Al{-}Shedivat, Bansal, Burda, Sutskever,
  Mordatch, and Abbeel]{al2017continuous}
Al{-}Shedivat, M., Bansal, T., Burda, Y., Sutskever, I., Mordatch, I., and
  Abbeel, P.
\newblock Continuous adaptation via meta-learning in nonstationary and
  competitive environments.
\newblock \emph{CoRR}, abs/1710.03641, 2017.

\bibitem[Baydin et~al.(2015)Baydin, Pearlmutter, and
  Radul]{baydin2015automatic}
Baydin, A.~G., Pearlmutter, B.~A., and Radul, A.~A.
\newblock Automatic differentiation in machine learning: a survey.
\newblock \emph{CoRR}, abs/1502.05767, 2015.

\bibitem[Dennis \& Mor{\'e}(1977)Dennis and Mor{\'e}]{dennis1977quasi}
Dennis, Jr, J.~E. and Mor{\'e}, J.~J.
\newblock Quasi-newton methods, motivation and theory.
\newblock \emph{SIAM review}, 19\penalty0 (1):\penalty0 46--89, 1977.

\bibitem[Finn et~al.(2017)Finn, Abbeel, and Levine]{finn2017}
Finn, C., Abbeel, P., and Levine, S.
\newblock Model-agnostic meta-learning for fast adaptation of deep networks.
\newblock In \emph{Proceedings of the 34th International Conference on Machine
  Learning, {ICML} 2017, Sydney, NSW, Australia, 6-11 August 2017}, pp.\
  1126--1135, 2017.

\bibitem[Foerster et~al.(2018)Foerster, Chen, Al{-}Shedivat, Whiteson, Abbeel,
  and Mordatch]{foerster2018}
Foerster, J.~N., Chen, R.~Y., Al{-}Shedivat, M., Whiteson, S., Abbeel, P., and
  Mordatch, I.
\newblock Learning with opponent-learning awareness.
\newblock In \emph{AAMAS}, 2018.

\bibitem[Fu(2006)]{fu2006gradient}
Fu, M.~C.
\newblock Gradient estimation.
\newblock \emph{Handbooks in operations research and management science},
  13:\penalty0 575--616, 2006.

\bibitem[Furmston et~al.(2016)Furmston, Lever, and
  Barber]{furmston2016approximate}
Furmston, T., Lever, G., and Barber, D.
\newblock Approximate newton methods for policy search in markov decision
  processes.
\newblock \emph{Journal of Machine Learning Research}, 17\penalty0
  (227):\penalty0 1--51, 2016.

\bibitem[Glynn(1990)]{glynn1990likelihood}
Glynn, P.~W.
\newblock Likelihood ratio gradient estimation for stochastic systems.
\newblock \emph{Communications of the ACM}, 33\penalty0 (10):\penalty0 75--84,
  1990.

\bibitem[Grondman et~al.(2012)Grondman, Busoniu, Lopes, and
  Babuska]{grondman2012survey}
Grondman, I., Busoniu, L., Lopes, G.~A., and Babuska, R.
\newblock A survey of actor-critic reinforcement learning: Standard and natural
  policy gradients.
\newblock \emph{IEEE Transactions on Systems, Man, and Cybernetics, Part C
  (Applications and Reviews)}, 42\penalty0 (6):\penalty0 1291--1307, 2012.

\bibitem[Kingma \& Welling(2013)Kingma and Welling]{kingma2013auto}
Kingma, D.~P. and Welling, M.
\newblock Auto-encoding variational bayes.
\newblock \emph{CoRR}, abs/1312.6114, 2013.

\bibitem[Li et~al.(2017)Li, Zhou, Chen, and Li]{li2017meta}
Li, Z., Zhou, F., Chen, F., and Li, H.
\newblock Meta-sgd: Learning to learn quickly for few shot learning.
\newblock \emph{CoRR}, abs/1707.09835, 2017.

\bibitem[Paszke et~al.(2017)Paszke, Gross, Chintala, Chanan, Yang, DeVito, Lin,
  Desmaison, Antiga, and Lerer]{paszke2017automatic}
Paszke, A., Gross, S., Chintala, S., Chanan, G., Yang, E., DeVito, Z., Lin, Z.,
  Desmaison, A., Antiga, L., and Lerer, A.
\newblock Automatic differentiation in pytorch.
\newblock 2017.

\bibitem[Pearlmutter(1994)]{pearlmutter1994fast}
Pearlmutter, B.~A.
\newblock Fast exact multiplication by the hessian.
\newblock \emph{Neural computation}, 6\penalty0 (1):\penalty0 147--160, 1994.

\bibitem[Ranganath et~al.(2014)Ranganath, Gerrish, and
  Blei]{ranganath2014black}
Ranganath, R., Gerrish, S., and Blei, D.~M.
\newblock Black box variational inference.
\newblock In \emph{Proceedings of the Seventeenth International Conference on
  Artificial Intelligence and Statistics, {AISTATS} 2014, Reykjavik, Iceland,
  April 22-25, 2014}, pp.\  814--822, 2014.

\bibitem[Rezende et~al.(2014)Rezende, Mohamed, and
  Wierstra]{rezende2014stochastic}
Rezende, D.~J., Mohamed, S., and Wierstra, D.
\newblock Stochastic backpropagation and approximate inference in deep
  generative models.
\newblock pp.\  1278--1286, 2014.

\bibitem[Schulman et~al.(2015)Schulman, Heess, Weber, and
  Abbeel]{schulman2015gradient}
Schulman, J., Heess, N., Weber, T., and Abbeel, P.
\newblock {Gradient Estimation Using Stochastic Computation Graphs}.
\newblock In \emph{Advances in Neural Information Processing Systems 28: Annual
  Conference on Neural Information Processing Systems 2015, December 7-12,
  2015, Montreal, Quebec, Canada}, pp.\  3528--3536, 2015.

\bibitem[Schulman et~al.(2017)Schulman, Abbeel, and
  Chen]{schulman2017equivalence}
Schulman, J., Abbeel, P., and Chen, X.
\newblock Equivalence between policy gradients and soft q-learning.
\newblock \emph{CoRR}, abs/1704.06440, 2017.

\bibitem[Stadie et~al.(2018)Stadie, Yang, Houthooft, Chen, Duan, Wu, Abbeel,
  and Sutskever]{stadie2018some}
Stadie, B., Yang, G., Houthooft, R., Chen, X., Duan, Y., Wu, Y., Abbeel, P.,
  and Sutskever, I.
\newblock Some considerations on learning to explore via meta-reinforcement
  learning, 2018.
\newblock URL \url{https://openreview.net/forum?id=Skk3Jm96W}.

\bibitem[Tucker et~al.(2017)Tucker, Mnih, Maddison, Lawson, and
  Sohl-Dickstein]{tucker2017rebar}
Tucker, G., Mnih, A., Maddison, C.~J., Lawson, J., and Sohl-Dickstein, J.
\newblock Rebar: Low-variance, unbiased gradient estimates for discrete latent
  variable models.
\newblock In \emph{Advances in Neural Information Processing Systems}, pp.\
  2624--2633, 2017.

\bibitem[Weaver \& Tao(2001)Weaver and Tao]{weaver2001optimal}
Weaver, L. and Tao, N.
\newblock The optimal reward baseline for gradient-based reinforcement
  learning.
\newblock In \emph{Proceedings of the Seventeenth conference on Uncertainty in
  artificial intelligence}, pp.\  538--545. Morgan Kaufmann Publishers Inc.,
  2001.

\bibitem[Williams(1992)]{williams1992simple}
Williams, R.~J.
\newblock Simple statistical gradient-following algorithms for connectionist
  reinforcement learning.
\newblock In \emph{Reinforcement Learning}, pp.\  5--32. Springer, 1992.

\bibitem[Wingate \& Weber(2013)Wingate and Weber]{wingate2013automated}
Wingate, D. and Weber, T.
\newblock Automated variational inference in probabilistic programming.
\newblock \emph{CoRR}, abs/1301.1299, 2013.

\end{thebibliography}
\bibliographystyle{icml2018}

\end{document}